\documentclass[sigconf]{acmart}

\usepackage[utf8]{inputenc}
\usepackage{mathtools}
\usepackage{bm}
\usepackage[ruled,vlined,linesnumbered]{algorithm2e}
\usepackage{booktabs} 
\usepackage{multirow}
\usepackage{tikz}
\usepackage{pifont} 
\usepackage{pgfplots}
\pgfplotsset{compat=1.18}
\usetikzlibrary{patterns}
\usepackage{tcolorbox}
\tcbuselibrary{skins, breakable}
\usepackage{fancyvrb}

\newtcolorbox{promptbox}[2][]{%
    enhanced,               
    colframe=gray!60!black,
    colback=gray!5,       
    coltitle=white,        
    colbacktitle=gray!70!black, 
    fonttitle=\bfseries,    
    arc=1mm, boxrule=0.5pt,
    titlerule=0pt,          
    left=5pt, right=5pt, top=5pt, bottom=5pt,
    title={#2},             
    breakable,             
    #1                    
}
\usetikzlibrary{shapes.geometric, arrows.meta, positioning, fit, backgrounds, calc, matrix, shadows} 
\definecolor{myred}{RGB}{250, 220, 220}  
\definecolor{myredborder}{RGB}{180, 50, 50}
\definecolor{myblue}{RGB}{220, 235, 250}  
\definecolor{myblueborder}{RGB}{40, 90, 160}
\definecolor{myyellow}{RGB}{255, 240, 180} 
\definecolor{myyellowborder}{RGB}{200, 160, 40}
\definecolor{mypink}{RGB}{250, 210, 210}   
\definecolor{mycyan}{RGB}{200, 230, 230}   
\DeclareMathOperator*{\argmax}{arg\,max}
\DeclareMathOperator*{\argmin}{arg\,min}
\newcommand{\E}{\mathbb{E}}
\newcommand{\x}{\mathbf{x}}

\newcommand{\C}{\mathcal{C}}
\newcommand{\Lcal}{\mathcal{L}}
\newcommand{\U}{\mathcal{U}}

\setcopyright{acmcopyright}
\copyrightyear{2026} 
\acmYear{2026} 
\acmDOI{10.1145/XXXXXXX.XXXXXXX} 
\acmConference[KDD '26]{Proceedings of the 32nd ACM SIGKDD Conference on Knowledge Discovery and Data Mining}{August 9--13, 2026}{Jeju, South Korea}
\acmBooktitle{Proceedings of the 32nd ACM SIGKDD Conference on Knowledge Discovery and Data Mining (KDD '26), August 9--13, 2026, Jeju, South Korea}

\begin{document}
\title{Reducing Hallucination in Enterprise AI Workflows via Hybrid Utility Minimum Bayes Risk (HUMBR)}

\author{Chenhao Fang}
\authornote{Authors contributed equally.}
\affiliation{%
  \institution{Meta Platforms, Inc.}
  \country{Menlo Park, CA, USA}
}
\email{chenhaofang@meta.com}

\author{Jordi Mola}
\authornotemark[1]
\affiliation{%
  \institution{Meta Platforms, Inc.}
  \country{Bellevue, WA, USA}
}
\email{jordim@meta.com}

\author{Mark Harman}
\affiliation{%
  \institution{Meta Platforms, Inc.}
  \country{London, UK}
}
\email{markharman@meta.com}

\author{Jason Nawrocki}
\affiliation{%
  \institution{Meta Platforms, Inc.}
  \country{Andover, MA, USA}
}
\email{jnawrocki@meta.com}

\author{Vaibhav Shrivastava}
\affiliation{%
  \institution{Meta Platforms, Inc.}
  \country{Bellevue, WA, USA}
}
\email{svaibhav@meta.com}

\author{Yue Cheng}
\affiliation{%
  \institution{Meta Platforms, Inc.}
  \country{Bellevue, WA, USA}
}
\email{yuecheng1118@meta.com}

\author{Jay Shah}
\affiliation{%
  \institution{Meta Platforms, Inc.}
  \country{New York, NY, USA}
}
\email{jaymshah@meta.com}

\author{Katayoun Zand}
\affiliation{%
  \institution{Meta Platforms, Inc.}
  \country{Menlo Park, CA, USA}
}
\email{katiezand@meta.com}

\author{Mansi Tripathi}
\affiliation{%
  \institution{Meta Platforms, Inc.}
  \country{Bellevue, WA, USA}
}
\email{mansitri@meta.com}

\author{Arya Pudota}
\affiliation{%
  \institution{Meta Platforms, Inc.}
  \country{Bellevue, WA, USA}
}
\email{arypu@meta.com}

\author{Matthew Becker}
\affiliation{%
  \institution{Meta Platforms, Inc.}
  \country{Bellevue, WA, USA}
}
\email{matthewbecker@meta.com}

\author{Hervé Robert}
\affiliation{%
  \institution{Meta Platforms, Inc.}
  \country{Menlo Park, CA, USA}
}
\email{hervert@meta.com}

\author{Abhishek Gulati}
\affiliation{%
  \institution{Meta Platforms, Inc.}
  \country{Menlo Park, CA, USA}
}
\email{akg@meta.com}

\renewcommand{\shortauthors}{Fang and Mola, et al.}

\begin{abstract}



Although LLMs drive automation, it is important to ensure immense consideration for high-stakes enterprise workflows such as those involving legal matters, risk management, and privacy compliance.

For Meta, and other organizations like ours, a single hallucinated clause in such high stakes workflows may introduce greater risks.

We show that by framing hallucination mitigation as a Minimum Bayes Risk (MBR) problem, we can reduce this risk.

Specifically, we introduce a Hybrid Utility MBR (HUMBR) framework that synthesizes semantic embedding similarity with lexical precision to identify consensus without ground-truth references, for which we  derive rigorous error bounds. 

We complement this theoretical analysis with a comprehensive empirical evaluation on widely-used public benchmark suites (TruthfulQA and LegalBench) and also real world data from Meta production deployment.

The results from our empirical study show that MBR significantly outperforms standard Universal Self-Consistency. Notably, 81\% of the pipeline's suggestions were preferred over human-crafted ground truth, and critical recall failures were also reduced.
\end{abstract}



\maketitle

\section{Introduction}
Many organizations, including Meta, are using
Large Language Models (LLMs) to catalyze a paradigm shift in the way information is processed: moving from rigid, rule-based automation to flexible, semantic-aware intelligence \cite{zhu2025compliancebrainassistantconversational, fang2025privacyartifactconnectorpact, fang2024ingestandgrounddispellinghallucinationscontinuallypretrained}. 
However, there is a well-known trust gap~\cite{Ji_2023} due to the inherent stochastic nature of generative AI.
In particular, AI hallucination presents a severe existential risk when AI is applied to high-stakes corporate workflows such as legal discovery, privacy compliance engineering, and regulatory interpretation.


The challenge is that single-model generations are inherently unstable; even State-Of-The-Art models  can confidently fabricate details when operating on long-tail knowledge or complex logical predicates \cite{zhang2025sirenssongaiocean}.
Standard industrial approaches to mitigate hallucination typically rely on Retrieval-Augmented Generation~\cite{ni2025trustworthyretrievalaugmentedgeneration}  or iterative refinement~\cite{dhuliawala2023chainofverificationreduceshallucinationlarge}; `Ask the model to critique itself'. 
Unfortunately, these methods have distinct failure modes in production. 
Relying on a single decoding path even with temperature $T=0$ leaves the system vulnerable to the specific biases and blind spots of that model instance. 
Another common ensemble technique is to generate multiple responses and ask an LLM to summarize them. 
But this is often flawed in high-precision tasks. 
Summarization introduces a second order of generation, creating a risk of compounding hallucinations where the summarizer conflates conflicting details or smooths over nuances needed for compliance \cite{turpin2023languagemodelsdontsay}.

To address the challenge of minimizing hallucination risk for high-stakes enterprise workflows, we draw inspiration from decision theory, introducing the \textbf{Minimum Bayes Risk (MBR)} approach to hallucination mitigation.
Our key insight is that, while individual models may hallucinate, they tend to hallucinate \textit{differently} \cite{manakul2023selfcheckgptzeroresourceblackboxhallucination,dhuliawala2023chainofverificationreduceshallucinationlarge}. 
True information acts as an attractor in the semantic space—diverse models and decoding paths tend to converge on the truth—whereas hallucinations are often stochastic outliers \cite{wang2023selfconsistencyimproveschainthought}. 
By generating an ensemble of candidates across heterogeneous models and temperatures, and selecting the centroid candidate that maximizes the utility function, we can mathematically filter out noise without needing a ground-truth reference. Our primary contributions are to:

\begin{itemize}
    \item introduce \textbf{Hybrid Utility MBR (HUMBR)}, a robust algorithm for LLM answer selection.
    \item \textbf{prove} that HUMBR  asymptotically approaches ground truth under reasonable sparse hallucination assumptions.
        
    \item show that HUMBR can be \textbf{optimally configured} to satisfy a strict hallucination tolerance.
    
    \item report results of a \textbf{comprehensive empirical study} using  public benchmarks and real-world deployment. 
    The results reveal that HUMBR \textbf{outperforms with statistical significance} (general and domain-specific) single-model baselines and standard ensemble method and also human experts.
\end{itemize}

\section{Related Work}

\begin{figure*}[t!]
    \centering
    \resizebox{\linewidth}{!}{
        \definecolor{acadiaBlue}{RGB}{0, 40, 85}
        \definecolor{techBlue}{RGB}{70, 130, 180}
        \definecolor{lightGrey}{RGB}{250, 250, 252} 
        \definecolor{borderGrey}{RGB}{180, 180, 190}
        \definecolor{highlightOrange}{RGB}{210, 80, 0}
        \definecolor{scoreCellBg}{RGB}{235, 240, 245}
        \begin{tikzpicture}[
            font=\sffamily,
            >=Latex,
            node distance=0.8cm,
            baseBlock/.style={rectangle, draw=borderGrey, thick, fill=white, align=center, minimum height=0.9cm},
            promptSty/.style={baseBlock, minimum width=1.5cm, fill=lightGrey},
            lmSty/.style={rectangle, draw=acadiaBlue, very thick, fill=techBlue!20, minimum width=2.0cm, minimum height=1.6cm, align=center, font=\bfseries, drop shadow},
            candidateSty/.style={rectangle, draw=borderGrey, thick, fill=white, minimum size=0.85cm, font=\bfseries\small},
            scoreSty/.style={rectangle, draw=borderGrey!50, fill=scoreCellBg, minimum width=0.9cm, minimum height=0.65cm, font=\footnotesize},
            totalSty/.style={rectangle, draw=acadiaBlue, thick, fill=techBlue!30, minimum width=0.9cm, minimum height=0.65cm, font=\bfseries\footnotesize},
            decisionSty/.style={diamond, draw=highlightOrange, thick, fill=highlightOrange!10, align=center, font=\footnotesize\itshape, aspect=1.3, inner sep=1.5pt},
            finalSty/.style={rectangle, draw=acadiaBlue, very thick, fill=white, minimum size=1.0cm, font=\bfseries},
            phaseLabel/.style={font=\bfseries\small, text=acadiaBlue, align=center, anchor=south} 
        ]
        \node[promptSty] (prompt) {Prompt\\ \scriptsize(Input)};
        \node[lmSty, right=of prompt] (lm) {LLMs};
        \draw[->, thick] (prompt) -- (lm);
        \node[candidateSty, right=1.6cm of lm, yshift=1.2cm] (c1) {$C_1$};
        
        \node[scoreSty, right=0.3cm of c1] (u11) {1.00};
        \node[scoreSty, right=0.1cm of u11] (u12) {0.85};
        \node[font=\footnotesize, right=0.1cm of u12] (udots1) {$\dots$};
        \node[scoreSty, right=0.1cm of udots1] (u1n) {0.88};
        \node[totalSty, right=0.5cm of u1n] (s1) {0.86};
        \node[candidateSty, below=0.4cm of c1] (c2) {$C_2$};
        \node[scoreSty, right=0.3cm of c2, draw=techBlue] (u21) {0.91};
        \node[scoreSty, right=0.1cm of u21, draw=techBlue] (u22) {1.00};
        \node[font=\footnotesize, right=0.1cm of u22] (udots2) {$\dots$};
        \node[scoreSty, right=0.1cm of udots2, draw=techBlue] (u2n) {0.94};
        \node[totalSty, right=0.5cm of u2n, fill=techBlue!50, thick] (s2) {\textbf{0.95}};
        \node[below=0.1cm of c2] (vdots) {$\vdots$};
        \node[below=0.1cm of s2] {$\vdots$};
        \node[candidateSty, below=0.4cm of vdots] (cn) {$C_N$};
        \node[scoreSty, right=0.3cm of cn] (un1) {0.40};
        \node[scoreSty, right=0.1cm of un1] (un2) {0.84};
        \node[font=\footnotesize, right=0.1cm of un2] (udotsn) {$\dots$};
        \node[scoreSty, right=0.1cm of udotsn] (unn) {1.00};
        \node[totalSty, right=0.5cm of unn] (sn) {0.75};
        \node[font=\footnotesize, below=0.1cm of un1] (lbl1) {$C_1$};
        \node[font=\footnotesize, below=0.1cm of un2] (lbl2) {$C_2$};
        \node[font=\footnotesize, below=0.1cm of unn] (lblN) {$C_N$};
        \node[font=\bfseries\footnotesize, below=0.1cm of sn, text=acadiaBlue, align=center] (lblScore) {Score $S_i$};
        \draw[->, dashed, thick, shorten >=2pt] (lm.east) -- (c1.west) node[midway, above, sloped, font=\tiny, yshift=1pt] {Sampling};
        \draw[->, dashed, thick, shorten >=2pt] (lm.east) -- (c2.west);
        \draw[->, dashed, thick, shorten >=2pt] (lm.east) -- (cn.west) node[midway, below, sloped, font=\tiny, yshift=-1pt] {$T<0.7$};
        \node[decisionSty, right=0.6cm of s2] (gate) {Check\\[1pt] $S_{max} > \tau$?};
        \node[finalSty, right=0.8cm of gate] (output) {$C_2$};
        \draw[->, very thick, techBlue] (s2.east) -- (gate.west);
        \draw[->, very thick, acadiaBlue] (gate.east) -- (output.west) node[midway, above, font=\tiny\bfseries] {PASS};
        \draw[->, thick, gray, dashed] (gate.south) -- ++(0, -0.6cm) node[below, font=\tiny] (abstainNode) {ABSTAIN};
        \node[above=1.1cm of u12, xshift=1.0cm, font=\scriptsize, align=center, draw=highlightOrange!60, fill=white, rounded corners, inner sep=6pt, drop shadow={opacity=0.1}] (formula) {
            \textbf{\textcolor{acadiaBlue}{Pairwise Hybrid Utility Function}} \\[3pt]
            $U_{ij} = \alpha \cdot \text{CosSim}(E_i, E_j) + (1-\alpha) \cdot \text{ROUGE}(C_i, C_j)$
        };
        \draw[->, borderGrey, shorten >=1pt] (formula.south) -- (u11.north);
        \draw[->, borderGrey, shorten >=1pt] (formula.south) -- (u1n.north);
        \node[phaseLabel, above=0.3cm of formula] (p2) {Phase 2:\\Consensus Calculation};
        
        \node[phaseLabel] at (c1 |- p2.south) [xshift=-1.0cm, anchor=south] (p1) {Phase 1:\\Generation};
        \node[phaseLabel] at (gate |- p2.south) [xshift=0cm, anchor=south] (p3) {Phase 3:\\Selection};
        
        \begin{pgfonlayer}{background}
            \node[rectangle, draw=borderGrey!80, rounded corners=10pt, fill=lightGrey!50, 
            inner sep=10pt, 
            fit=(prompt)(output)(formula)(p2)(p1)(p3)(lbl1)(lblScore)(abstainNode), 
            label={[yshift=-0.5em, font=\bfseries\small\color{acadiaBlue}]south:HUMBR Consensus Framework}] (mainBox) {};
        \end{pgfonlayer}
        \draw[borderGrey, dashed] ($(c1.west)!0.5!(lm.east)$ |- mainBox.north) -- ($(c1.west)!0.5!(lm.east)$ |- mainBox.south);
        \draw[borderGrey, dashed] ($(s1.east)!0.5!(gate.west)$ |- mainBox.north) -- ($(sn.east)!0.5!(gate.west)$ |- mainBox.south);
        \end{tikzpicture}
    }
    \caption{\textbf{Workflow of the proposed HUMBR ensemble system.} The system calculates consensus using a Hybrid Utility Function (Semantic + Lexical) and applies a consensus threshold $\tau$ to filter risky outputs.}
    \label{fig:method_workflow}
\end{figure*}
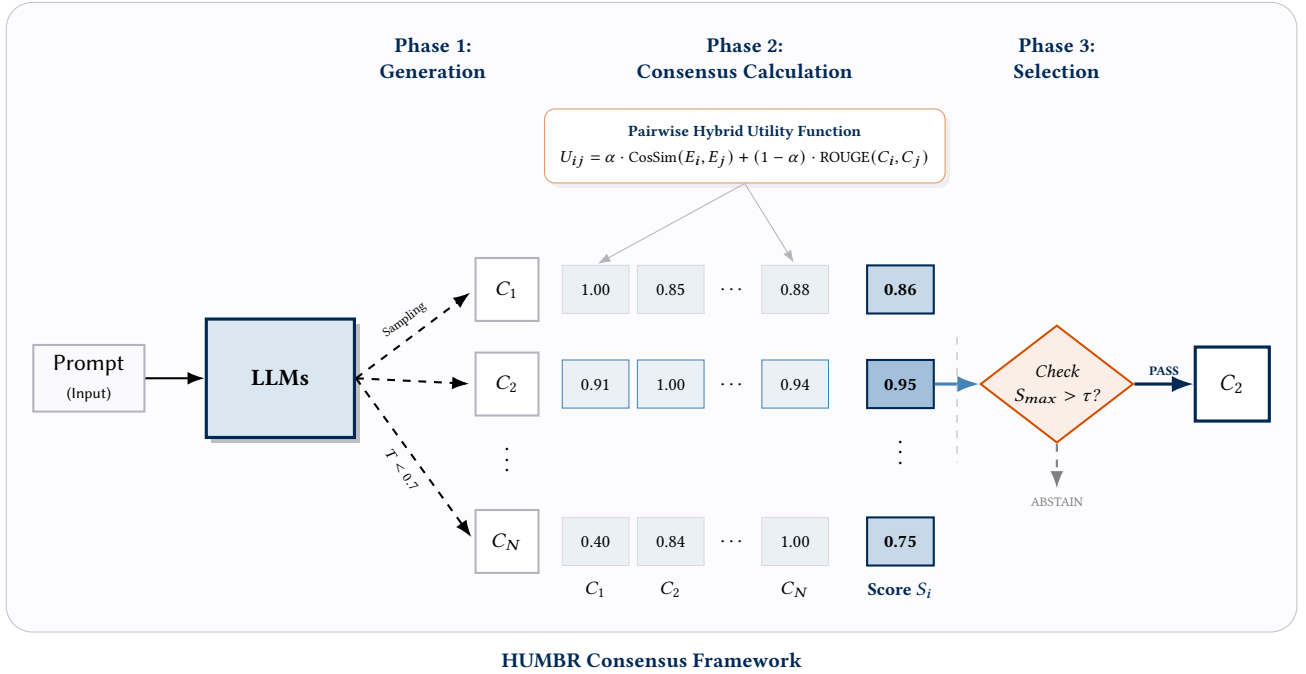

\subsection{Hallucination Mitigation in LLMs}
``Hallucination'', the generation of plausible but factually incorrect content, remains a major barrier to enterprise LLM adoption \cite{Ji_2023}, typically categorized into intrinsic (contradicting inputs) and extrinsic (fabricating details) types \cite{zhang2025sirenssongaiocean}. Standard mitigation strategies operate at different stages of the model lifecycle: RLHF \cite{ouyang2022traininglanguagemodelsfollow} aligns models during training, while Retrieval-Augmented Generation (RAG) \cite{lewis2021retrievalaugmentedgenerationknowledgeintensivenlp} grounds inference in external evidence. However, generators may still fail due to attention lapses or context loss \cite{liu2023lostmiddlelanguagemodels}. Consequently, post-hoc verification has become critical, with approaches like \textit{SelfCheckGPT} \cite{manakul2023selfcheckgptzeroresourceblackboxhallucination} exploiting the stochastic instability of hallucinations to detect inconsistencies.

\subsection{Advanced Decoding and Inference-Time Interventions}
Recent research manipulates the decoding process to enhance factuality. Techniques like \textit{Inference-Time Intervention (ITI)} \cite{li2024inferencetimeinterventionelicitingtruthful} and \textit{DoLa} \cite{chuang2024doladecodingcontrastinglayers} shift activation heads or contrast layers to amplify truthful signals. Similarly, \textit{Context-Aware Decoding (CAD)} \cite{shi2023trustingevidencehallucinatecontextaware} penalizes prior probabilities to enforce reliance on provided context. While effective, these token-level interventions often require task-specific tuning. Our HUMBR approach offers a robust, sequence-level alternative that optimizes for holistic consensus rather than local token probabilities.

\subsection{Ensemble Learning and Consensus}
Ensemble methods leverage diversity to improve robustness. \textit{Self-Consistency} \cite{wang2023selfconsistencyimproveschainthought} demonstrated that majority voting boosts chain-of-thought performance, though naive voting is brittle for open-ended generation. To address this, Fang et al. \cite{fang2024llmensembleoptimallargelanguage} proposed \textit{LLM-Ensemble}, a probabilistic framework for weighted consensus that treats generation as a truth discovery problem.

Alternative consensus strategies include \textit{Multi-Agent Debate} \cite{du2023improvingfactualityreasoninglanguage, liang2024encouragingdivergentthinkinglarge}, which induces agreement through dialogue, and \textit{Generate-then-Rank} pipelines like \textit{LLM-Blender} \cite{jiang2023llmblenderensemblinglargelanguage} or trained Verifiers \cite{cobbe2021trainingverifierssolvemath}. However, these often suffer from high latency or require domain-specific training. Our approach distinguishes itself as reference-free, unsupervised, and single-turn, using embedding similarity to find a ``centroid'' consensus without auxiliary models \cite{chen2023universalselfconsistencylargelanguage}.

\subsection{Minimum Bayes Risk (MBR) Decoding}
Originally from Statistical Machine Translation \cite{goel2002minimum}, MBR has replaced MAP decoding in contexts where probability concentration leads to repetitive outputs \cite{eikema2020map}. Recently readapted for LLMs, MBR has been applied to instruction following \cite{suzgun2023follow} and translation \cite{freitag-etal-2023-epsilon} using metrics like BERTScore or COMET. While Kuhn et al. \cite{kuhn2023semanticuncertaintylinguisticinvariances} explored similar ``semantic uncertainty,'' existing works typically rely on single metrics. We advance this by introducing a Hybrid Utility Function, combining semantic embeddings with lexical ROUGE scores to ensure both conceptual accuracy and structural coherence.

\section{Methodology}
\label{sec:method}
Standard autoregressive decoding implicitly assumes that the mode of the distribution corresponds to factual truth. However, in high-entropy open-domain generation, this assumption frequently fails as probability mass often concentrates on generic or hallucinated sequences. To address this, we reframe hallucination mitigation not as a likelihood maximization task, but as a \textbf{Minimum Bayes Risk (MBR)} decision problem. Our core hypothesis is geometric: while hallucinations are stochastically dispersed, factual answers tend to cluster in a shared semantic neighborhood. Therefore, the truth manifests as the centroid of the semantic probability mass, rather than the sharpest peak of the density function.

The overall architecture of our system is illustrated in Figure \ref{fig:method_workflow}. The workflow proceeds in three phases: diverse sampling to approximate the posterior, pairwise consensus calculation, and risk-aware selection. We structure this section as follows:

\begin{itemize}
    \item We first establish the theoretical foundations of reference-free MBR in \textbf{Section \ref{subsec:mbr_theory}}.
    \item Moving from theory to reality, \textbf{Section \ref{subsec:risk_bounds}} derives a generalized error bound that explicitly accounts for intra-model correlations.
    \item Leveraging these bounds, \textbf{Section \ref{subsec:system_design}} solves the engineering optimization problem, determining the minimal ensemble size required to meet a strict quality threshold.
    \item Finally, \textbf{Section \ref{subsec:implementation}} details the practical implementation, introducing our Hybrid Utility Function and the complete HUMBR algorithm.
\end{itemize}

\subsection{Minimum Bayes Risk Formulation}
\label{subsec:mbr_theory}

Let $P(y|x)$ denote the true posterior distribution of valid responses $y$ given an input prompt $x$. In generative tasks, our objective is to select a hypothesis $\hat{y}$ from a hypothesis space $\mathcal{H}$ that minimizes the expected loss relative to the true distribution.

\begin{definition}[Bayes Risk]
Given a loss function $\Lcal(y, \hat{y}): \mathcal{Y} \times \mathcal{Y} \to \mathbb{R}_{\ge 0}$ which quantifies the divergence between a reference $y$ and a hypothesis $\hat{y}$, the Bayes Risk $R(\hat{y}|x)$ is defined as the expected loss under the posterior:
\begin{equation}
    R(\hat{y}|x) = \E_{y \sim P(\cdot|x)} \left[ \Lcal(y, \hat{y}) \right] = \sum_{y \in \mathcal{Y}} P(y|x) \Lcal(y, \hat{y})
\end{equation}
\end{definition}

The optimal selection strategy, known as MBR selection, selects the candidate that minimizes this risk:
\begin{equation}
    \hat{y}_{\text{MBR}} = \argmin_{\hat{y} \in \mathcal{H}} R(\hat{y}|x)
\end{equation}

In reference-free settings where the true posterior $P(y|x)$ is intractable, we approximate the expectation using Monte Carlo integration. We assume the ensemble of generated candidates $\C = \{c_1, \dots, c_N\}$ are independent samples drawn from the model's approximate posterior $P_{\theta}(y|x)$. By defining a utility function $\U(y, \hat{y}) = 1 - \Lcal(y, \hat{y})$, minimizing risk becomes equivalent to maximizing expected utility:
\begin{equation}
    \hat{y}_{\text{MBR}} \approx \argmax_{\hat{y} \in \C} \frac{1}{N} \sum_{c_i \in \C} \U(c_i, \hat{y})
    \label{eq:centroid}
\end{equation}

Mathematically, $\hat{y}_{\text{MBR}}$ maximizes the lower bound of the expected utility under the true distribution, thus selects the optimal candidate. The full proof is provided in Appendix \ref{app:proof_theorem_1}.

\subsection{Risk Bounds under Intra-Model Correlation}
\label{subsec:risk_bounds}

While MBR maximizes expected utility, practical deployment in zero-tolerance domains requires a rigorous quality guarantee. Specifically, we must bound the probability that the system confidently outputs a hallucination.

A limitation of prior analysis is the assumption of independence. In an ensemble of $K$ models generating $M$ samples each (total $N=KM$), samples from the same model exhibit strong correlations ($\rho$). We derive a generalized error bound that explicitly accounts for this correlation using a hierarchical framework.

Let $Z_k$ denote the effective count of divergent (hallucinated) samples produced by model $L_k$. While the utility space is continuous, for the purpose of a \textit{worst-case risk analysis}, we discretize the outcome into binary states: a sample is either \textit{consistent} with the consensus mode or \textit{divergent}.

To capture the intra-model dependency, we model the divergent count $Z_k$ via a hierarchical Beta-Binomial process:
\begin{equation}
    \begin{aligned}
        \pi_k &\sim \text{Beta}(a, b) \\
        Z_k | \pi_k &\sim \text{Binomial}(M, \pi_k)
    \end{aligned}
\end{equation}

where $\pi_k$ is the latent error rate of model $k$. The marginal distribution of $Z_k$ is governed by the mean error rate $\mu$ and the correlation coefficient $\rho$:
\begin{equation}
    P(Z_k=z) = \binom{M}{z} \frac{B(z + \mu(\frac{1}{\rho}-1), M - z + (1-\mu)(\frac{1}{\rho}-1))}{B(\mu(\frac{1}{\rho}-1), (1-\mu)(\frac{1}{\rho}-1))}
\end{equation}

We introduce a global \textbf{Consensus Threshold} $\tau \in [0.5, 1]$. The system is considered valid only if the consensus score $S$ exceeds $\tau$.

In the worst-case analysis, a failure occurs when the hallucinations cluster so densely that their aggregated mass exceeds this threshold, tricking the MBR algorithm into selecting a falsehood.

Let $\mathbf{z} = (z_1, \dots, z_K)$ be the vector of divergent counts across $K$ models. The \textbf{Failure Region} $\Omega_{\mathcal{F}}(\tau)$ is defined as the state space where the total divergent mass constitutes a super-majority:
\begin{equation}
    \Omega_{\mathcal{F}}(\tau) = \left\{ \mathbf{z} \in \mathbb{Z}_{\ge 0}^K \;\Bigg|\; \frac{\sum_{k=1}^K z_k}{N} \ge \tau \right\}
\end{equation}

\begin{theorem}[Threshold-Aware Risk Bound]
\label{thm:error_bound}
For an ensemble operating under threshold $\tau$, the probability of a failure is the cumulative probability mass of the Beta-Binomial product distribution over $\Omega_{\mathcal{F}}(\tau)$:
\begin{equation}
    P_{\text{fail}} = \sum_{\mathbf{z} \in \Omega_{\mathcal{F}}(\tau)} \left[ \prod_{k=1}^K P_{\text{BetaBinomial}}(z_k; M, \mu_k, \rho_k) \right]
\end{equation}
\end{theorem}

By increasing $\tau$ (e.g., from 0.5 to 0.7), we shrink the volume of $\Omega_{\mathcal{F}}$, exponentially reducing the risk of accepting a hallucination, albeit at the cost of abstaining on ambiguous queries.

\subsection{Enterprise System Design: Meeting Quality Constraint}
\label{subsec:system_design}

Armed with the generalized error bound derived in Theorem \ref{thm:error_bound}, we address the inverse engineering problem: \textit{How to configure the ensemble to guarantee a specific error bound $\epsilon$ with minimum computational cost?}

\subsubsection{Optimization Formulation}
We formulate this as a constrained optimization problem. Let $C_k$ be the inference cost per token for model $k$, our objective is to find the optimal subset of models $\mathcal{K}$ and generation count $M$ such that the failure probability remains below an error tolerance $\epsilon$ (e.g., $\epsilon=10^{-4}$):
\begin{equation}
\begin{aligned}
    & \min_{\mathcal{K}, M} \quad \sum_{k \in \mathcal{K}} C_k \cdot M \\
    & \text{s.t.} \quad P_{\text{fail}}(\Omega_{\mathcal{F}}(\tau); M, \mu_k, \rho_k) \le \epsilon
\end{aligned}
\end{equation}

To intuitively understand the operational landscape defined above, we visualize the idealized \textbf{Cost-$P_{\text{fail}}$ Pareto Frontier} in Figure \ref{fig:pareto}. The plot reveals distinct regimes:

\begin{enumerate}
    \item \textbf{High Failure Rate Zone:} Single-model baselines (lower left) incur minimal computational cost but suffer from stochastic hallucinations ($P_{\text{fail}} \gg \epsilon$), making them unsuitable for compliance tasks.
    \item \textbf{Inefficient Zone:} Naive ensemble methods like "Summarization" increase computational cost linearly but often fail to reduce the failure probability proportionally due to compounding errors.
    \item \textbf{Feasible Zone:} Our MBR approach (green centroid) effectively navigates this trade-off. By leveraging the diversity of the ensemble ($\rho < 1$), it pushes the system's hallucination probability exponentially downward into the Target Zone (satisfying the constraint $P_{\text{fail}} \le \epsilon$) while maintaining a tractable inference budget.
\end{enumerate}

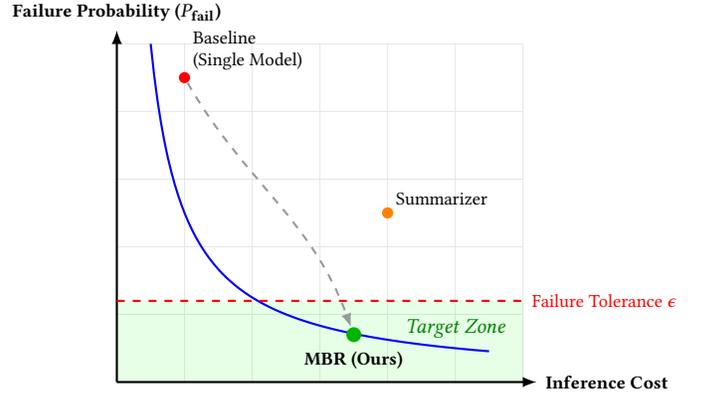
\begin{figure}[t]
\centering
\begin{tikzpicture}[scale=0.9, >=latex]
    \draw[step=1cm, gray!20, very thin] (0,0) grid (6,5);
    \fill[green, opacity=0.1] (0,0) rectangle (6,1.2);
    \node[green!50!black, font=\small\itshape] at (5, 0.8) {Target Zone};
    \draw[->, thick] (0,0) -- (6.2,0) node[right] {\footnotesize \textbf{Inference Cost}};
    \draw[->, thick] (0,0) -- (0,5.2) node[above] {\footnotesize \textbf{Failure Probability ($P_{\text{fail}}$)}};
    \draw[thick, blue, domain=0.5:5.5, samples=100, variable=\x] plot (\x, {2.5/\x});
    
    \draw[dashed, red, thick] (0, 1.2) -- (6, 1.2) node[right, font=\footnotesize] {Failure Tolerance $\epsilon$};
    
    \node[circle, fill=red, inner sep=1.5pt] (base) at (1, 4.5) {};
    \node[above right, align=left, font=\footnotesize] at (1, 4.5) {Baseline\\(Single Model)};
    \node[circle, fill=orange, inner sep=1.5pt] (summ) at (4, 2.5) {};
    \node[above right, align=left, font=\footnotesize] at (4, 2.5) {Summarizer};
    \node[circle, fill=green!70!black, inner sep=2pt] (mbr) at (3.5, 0.7) {};
    \node[below, font=\bfseries\footnotesize, align=center, yshift=-3pt] at (3.5, 0.7) {MBR (Ours)};
    \draw[->, thick, dashed, gray!80] (base) to[out=-60, in=110] (mbr);
\end{tikzpicture}
\caption{\textbf{Cost-$P_{\text{fail}}$ Pareto Frontier:} An idealized depiction of the operational trade-off landscape. }
\label{fig:pareto}
\end{figure}

\subsubsection{Derivation of Required Sample Size}
To solve this efficiently without brute-force simulation, we derive a closed-form approximation for the required number of generations $M$. While the ensemble contains heterogeneous models, for the purpose of deriving a conservative design bound, we model the system using effective average parameters. We assume an ensemble of $K$ models with an average error rate $\bar{\mu}$ and an effective pairwise correlation $\bar{\rho}$. The variance of the hallucination ratio $\hat{p} = Z/N$ is inflated by the variance inflation factor (VIF):
\begin{equation}
    \text{Var}(\hat{p}) = \frac{\bar{\mu}(1-\bar{\mu})}{N} \cdot [1 + (M-1)\bar{\rho}]
\end{equation}

Thus, the ensemble behaves as if it has an effective sample size of:
\begin{equation}
    N_{eff} = \frac{K \cdot M}{1 + (M-1)\bar{\rho}}
\label{eq:Neff}
\end{equation}

To ensure the hallucination rate stays below the error tolerance $\epsilon$, we apply a Hoeffding-type bound. We require the probability that the observed hallucination rate $\hat{p}$ exceeds the consensus threshold $\tau$ to be bounded by $\epsilon$:
\begin{equation}
    P_{\text{fail}} \le \exp \left( -2 N_{eff} \left( \tau - \bar{\mu} \right)^2 \right) \le \epsilon
\end{equation}

where $\bar{\mu}$ is the weighted average error rate of the ensemble. Solving for $M$, we obtain the design inequality:
\begin{equation}
    M \ge \frac{\ln(1/\epsilon)(1-\bar{\rho})}{2 K (\tau - \bar{\mu})^2 - \bar{\rho} \ln(1/\epsilon)}
    \label{eq:design_M}
\end{equation}

\subsubsection{Stratification and Diversity Strategy}
\label{subsub:temp}

Equation \ref{eq:Neff} reveals operational insights regarding the effective sample size $N_{eff}$:

\begin{itemize}
    \item \textbf{Independence ($\bar{\rho} \to 0$):} $N_{eff} = KM$, the full ensemble contributes.
    \item \textbf{Collapse ($\bar{\rho} \to 1$):} $N_{eff} \to K$, adding more samples per model yields zero gain.
    \item \textbf{Saturation:} As $M \to \infty$, $N_{eff} \to K/\bar{\rho}$, 
          proving that model diversity (increasing $K$) is essential.
\end{itemize}

\paragraph{Illustrative Example.}
Consider $K=4$ models, $M=4$ samples each, with base error rate $\bar\mu=0.1$ 
and consensus threshold $\tau=0.7$. Under independence ($\bar\rho=0$), 
the failure probability is $\approx 10^{-8}$. However, with realistic 
correlation $\bar\rho=0.5$, the risk rises to $\approx 10^{-4}$---a gap of 
\textbf{four orders of magnitude} that naive analysis would miss.

The intra-model correlation $\bar\rho$ is strictly dependent on the sampling strategy. Repeated sampling at a single low temperature yields high correlation ($\bar\rho \to 1$), while sampling at very high temperatures degrades quality ($\bar\mu$).

To optimally balance this trade-off, we propose a \textbf{Temperature Stratification} strategy. Instead of fixing a single temperature $T$, we enforce diversity by sampling along a temperature gradient. For example, in our internal deployment (Sec. \ref{subsec:internal_case}), to achieve an error tolerance of $\epsilon = 10^{-4}$, our optimization dictated a heterogeneous mix of $K=4$ models. Rather than generating $M$ i.i.d. samples, each model generates exactly one sample at four distinct temperature levels: $\mathcal{T} = \{0, 0.25, 0.5, 0.75\}$ (i.e., $M=4$). This design forces decorrelation within each model's outputs (minimizing $\bar\rho$) while retaining high-probability modes at lower temperatures (maintaining low $\bar\mu$). 

\subsection{Hybrid Utility MBR (HUMBR)}
\label{subsec:implementation}

With the ensemble generation protocol established to maximize diversity, we obtain a candidate set $\mathcal{C}$ that is robust against correlated hallucinations. The final operational step is to identify the consensus centroid within this pool.

We implement the MBR selection mechanism using a \textbf{Hybrid Utility Function}. This function is designed to measure consensus not just by exact string matching which is too brittle for diverse temperatures, but by capturing both conceptual meaning and structural phrasing. We define the utility between candidates $c_i, c_j$ as:
\begin{equation}
    \U(c_i, c_j) = \alpha \cdot \text{CosSim}(\phi(c_i), \phi(c_j))  +  (1-\alpha) \cdot \text{ROUGE-L}(c_i, c_j)
\end{equation}

where $\phi(\cdot)$ is a sentence embedding model and $\alpha$ is a weighting hyperparameter (typically 0.6). The complete selection procedure is formalized in Algorithm \ref{alg:mbr}.

Note that we introduce the threshold $\tau$. Unlike standard decoding which forces an output, our HUMBR formulation includes an \textit{abstention mechanism}: if the maximum consensus score $S_{k^*}$ falls below $\tau$, the system flags the query as ambiguous rather than risking a hallucination, thus enables a controllable trade-off between quality (precision) and coverage (recall). By demanding a super-majority (e.g., $\tau=0.75$), we can effectively reject ambiguous cases. 

In production, the HUMBR process naturally yields metadata for monitoring model health. By tracking the \textit{Divergence Score} (average distance from the consensus centroid) for each model in production, we can dynamically estimate the operational error rate $\hat{\mu}_k$ and self-correlation $\hat{\rho}_k$. Models that consistently deviate from the ensemble consensus are flagged for retraining or down-weighting in the voting logic.

\section{Offline Experiments}
To rigorously evaluate the proposed HUMBR framework, we conducted experiments to answer two Research Questions (RQs), identifying whether our method can discern truth from plausible mimicry:

\begin{itemize}
    \item \textbf{RQ1 Mitigation of Imitative Falsehoods:} Can HUMBR effectively filter out imitative falsehoods (common misconceptions) in open-domain scenarios?
    \item \textbf{RQ2 Domain Precision \& Robustness:} In high-stakes legal reasoning, does HUMBR outperform standard decoding strategies?
\end{itemize}

\subsection{Experimental Setup}

\paragraph{Datasets.} We utilize two distinct benchmarks to test breadth and depth:

\begin{enumerate}
    \item \textbf{TruthfulQA} (Generation Track): A benchmark comprising 817 questions designed to elicit imitative falsehoods. We focus on the \textit{Generation} task to test open-ended robustness.
    \item \textbf{LegalBench:} A legal reasoning benchmark. We focus on two representative task categories that challenge LLM reasoning:
    \begin{itemize}
        \item \textit{Interpretation (Classification):} Five difficult tasks requiring the classification of contractual clauses (e.g., \texttt{contract\_nli}).
        \item \textit{Rule-Application (Extraction/Generation):} Two tasks applying specific rules to facts (e.g., \texttt{hearsay}, \texttt{rule\_recall}).
    \end{itemize}
\end{enumerate}

\paragraph{Baselines.} We focus on \textit{reference-free, black-box} methods that are deployable via API access. Token-level interventions such as DoLa \cite{chuang2024doladecodingcontrastinglayers} and ITI \cite{li2024inferencetimeinterventionelicitingtruthful} require access to model internals (logits or activations), making them inapplicable to proprietary model APIs. We compare our approach against:

\begin{itemize}
    \item \textbf{Greedy Decoding:} Standard generation ($T=0$).
    \item \textbf{Universal Self-Consistency (USC) \cite{chen2023universalselfconsistencylargelanguage}:} The industry standard for reasoning tasks which leverages LLMs themselves to select the most consistent answer among multiple candidates.
    \item \textbf{Oracle (Best-of-N):} A ``cheating'' upper bound that has access to the gold reference at selection time. It selects the candidate from the pool $\mathcal{C}$ that best matches the ground truth. To avoid circular evaluation---where MBR's own utility metric would trivially favor itself---the Oracle employs a \textit{distinct} similarity metric: For classification tasks, we use \textit{exact label matching} after normalizing model outputs. For generation tasks, we use \textit{character-level sequence matching}, which computes the ratio of matching subsequences. This Oracle represents the maximum achievable performance given the candidate pool quality.
\end{itemize}

\paragraph{Metrics.} We employ a dual-metric strategy to ensure both semantic flexibility and domain rigor:

\begin{itemize}
    \item \textbf{Open-Domain (LLM-as-a-Judge):} For TruthfulQA, we use \texttt{Claude-Opus-4.5} to evaluate \% Truth $\times$ Info and a scalar Semantic Quality Score (1-100). To ensure evaluation robustness, we validated the LLM judge against both itself and human annotations. Details are provided in Appendix~\ref{app:judge_validation}.
    \item \textbf{Legal Domain (Exact Metrics):} For LegalBench, we report Balanced Accuracy for classification tasks and F1 Score for generative rule application, measuring the precision of legal terminology usage.
\end{itemize}

\paragraph{Experimental Setup.}
We generated a candidate pool of size $N=8$ per query using a stratified heterogeneous ensemble comprising equal splits from \texttt{Llama-3.3-70B} and \texttt{Gemini-2.0-Flash}. Sampling was conducted via stochastic decoding at temperature $T=0.7$ to maximize semantic entropy. For the HUMBR consensus mechanism, we set the hybrid utility weight $\alpha=0.6$, utilizing the open-source \texttt{DRAMA} \cite{ma2025drama} model for semantic embedding computation. 

While our production system (Section \ref{subsec:internal_case}) employs a strict consensus threshold ($\tau = 0.8$) to filter hallucinations, for these offline benchmark experiments, we relaxed the constraint to $\tau=0$. This mandates 100\% coverage, ensures a rigorous, unconstrained comparison against baselines, which typically lack intrinsic refusal mechanisms. Full details regarding prompt templates and judge specifications are provided in Appendix \ref{app:prompts}.

\subsection{RQ1: Mitigation of Imitative Falsehoods}
To answer whether geometric consensus can distinguish factual truth from popular misconceptions, we evaluate performance on the \textit{TruthfulQA} (Generation) benchmark.

\begin{table}[h]
    \centering
    \caption{\textbf{Performance on TruthfulQA.} Evaluated by Claude-Opus-4.5 Judge. \textit{\% Truth $\times$ Info} is the strict metric requiring answers to be both factually correct and informative (non-refusals).}
    \resizebox{\linewidth}{!}{%
        \begin{tabular}{l c c c}
            \toprule
            \textbf{Decoding Method} & \textbf{\% Truth $\times$ Info} & \textbf{Quality Score} \\
            \midrule
            Greedy Decoding ($T=0$) & 69.5\% & 74.1 \\
            Universal Self-Consistency (USC) & 76.5\% & 75.5 \\
            \midrule
            \textbf{HUMBR (Ours)} & \textbf{80.3\%} & \textbf{81.8} \\
            \textit{Oracle (Upper Bound)} & \textit{81.5\%} & \textit{83.5} \\
            \bottomrule
        \end{tabular}
    }
    \label{tab:truthfulqa}
\end{table}

\paragraph{Results \& Analysis.}
Table \ref{tab:truthfulqa} presents the comparative results. Our HUMBR approach achieves a \textbf{+3.8\%} absolute improvement over the Universal Self-Consistency (USC) baseline.

Notably, our method recovers approximately \textbf{98.5\%} of the Oracle performance (80.3\% vs 81.5\%), suggesting that the selection mechanism is near-optimal given the candidate pool.

\subsection{RQ2: Precision in Legal Domain}
Legal reasoning tolerates low hallucinations. A model must not only retrieve the correct rule but apply it using precise terms. In this section, we evaluate performance on the \textit{Interpretation} and \textit{Rule-Application} subsets of LegalBench.

\paragraph{Results.} Table \ref{tab:legalbench_results} compares our HUMBR against standard decoding and Universal Self-Consistency (USC).

\begin{table}[h]
    \centering
    \caption{\textbf{LegalBench Performance.} HUMBR outperforms baselines on both classification (Interpretation) and generation (Rule-Application) tasks.}
    \resizebox{\linewidth}{!}{%
        \begin{tabular}{l cc | cc}
            \toprule
            & \multicolumn{2}{c}{\textbf{Interpretation}} & \multicolumn{2}{c}{\textbf{Rule-Application}} \\
            \cmidrule(lr){2-3} \cmidrule(lr){4-5}
            \textbf{Method} & \textbf{Acc} & \textbf{Bal-Acc} & \textbf{Exact Match} & \textbf{F1 Score} \\
            \midrule
            Greedy Decoding & 35.1\% & 34.0\% & 44.4\% & 42.9\% \\
            Self-Consistency (USC) & 39.2\% & 36.5\% & 49.3\% & 45.5\% \\
            \midrule
            \textbf{HUMBR (Ours)} & \textbf{54.2\%} & \textbf{50.2\%} & \textbf{53.5\%} & \textbf{50.8\%} \\
            \textit{Oracle (Upper Bound)}  & 69.7\% & 67.1\% & 59.0\% & 53.9\% \\
            \bottomrule
        \end{tabular}%
    }
    \label{tab:legalbench_results}
\end{table}

\paragraph{Analysis.} 
We observe that HUMBR achieves a \textbf{+15.0\%} gain in Balanced Accuracy over USC. In legal classification tasks (e.g., determining if a clause is void), USC often suffers from "modal collapse" where the model confidently predicts the majority class due to prior bias. MBR, by weighing the semantic centrality of the generated reasoning paths, effectively filters out these high-confidence but shallow biases.

Furthermore, in \textit{Rule-Application}, the boost in F1 Score (\textbf{+7.9}\%) indicates that HUMBR preserves the precise legal wording better. While USC averages out the noise, HUMBR actively selects the candidate that best represents the consensus logic, reducing the risk of omitting key legal qualifiers.

\section{Online Deployment: Automated Regulatory Understanding Pipeline at Meta}
\label{subsec:internal_case}

Following the robust performance of HUMBR on the offline benchmarks, we advanced the system to a live production deployment. We deployed the framework within a regulatory understanding workflow at Meta. This workflow involves parsing complex legal statutes and generating precise, interpretable specifications for engineering teams. Unlike general-purpose chat, this domain requires strict adherence to source texts; a single ``hallucinated'' obligation or missed exemption can lead to compliance risks.

\subsection{Operational Workflow: AI as a Drafting Partner}
The system is deployed not as an autonomous agent, but as a high-precision suggestion engine designed to augment the capabilities of human interpretation team. In this Human-in-the-Loop workflow (Figure \ref{fig:hitl_workflow}), the AI acts as a drafting partner to reduce the cognitive load of the "zero-to-one" specification process.

When new regulatory text is ingested, the HUMBR pipeline processes the content and generates a \textit{Candidate Interpretation} (the suggestion). This suggestion is presented to human subject matter experts (SMEs) alongside the source text. The expert's role shifts from drafting specifications from scratch to reviewing, validating, and refining the AI's suggestions. This workflow ensures that while the AI accelerates the extraction of complex logic, the final sign-off authority remains strictly with human experts.

Consequently, the evaluation presented in this section assesses the quality of these \textit{AI Suggestions} against interpretations historically drafted entirely by humans. The goal is to determine if the AI suggestions achieve a level of fidelity that allows them to serve as a reliable starting point for high-stakes compliance engineering.

\begin{figure}[h]
    \centering
    \resizebox{0.9\linewidth}{!}{
    \begin{tikzpicture}[
        font=\sffamily\small,
        >=Latex,
        node distance=0.8cm,
        process/.style={rectangle, draw=black!60, thick, fill=white, rounded corners=2pt, minimum height=0.8cm, align=center},
        human/.style={rectangle, draw=blue!80!black, thick, fill=blue!10, rounded corners=2pt, minimum height=0.8cm, align=center, font=\bfseries},
        arrow/.style={->, thick, color=black!70}
    ]
    \node[process] (input) {Regulatory\\Text};
    \node[process, fill=green!10, right=0.8cm of input] (ai) {HUMBR\\Pipeline};
    \node[process, dashed, right=0.8cm of ai] (sugg) {\textit{AI Suggestion}\\(Draft)};
    \node[human, right=0.8cm of sugg] (human) {Human Expert\\Review \& Revision};
    \node[process, right=0.8cm of human] (final) {Final\\Spec};
    \draw[arrow] (input) -- (ai);
    \draw[arrow] (ai) -- (sugg);
    \draw[arrow] (sugg) -- (human);
    \draw[arrow] (human) -- (final);
    \end{tikzpicture}
    }
    \caption{\textbf{Human-in-the-Loop Workflow.} The AI generates a suggestion which serves as a draft for the Human Expert.}
    \label{fig:hitl_workflow}
\end{figure}
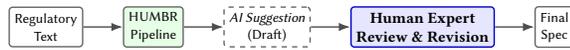

\subsection{Pipeline Setup}
We constructed a heterogeneous ensemble of four common Large Language Models, comprising both proprietary and open-weights models (\texttt{ Gemini-2.5-Pro, Gemini-2.5-Flash, Llama-3.3-70B, Llama-4-Maverick-17B-128E}) with distinct training corpus. As shown in Section \ref{subsub:temp}, to capture diverse reasoning paths, each model generated responses at four distinct temperature settings $\{0, 0.25, 0.5, 0.75\}$, resulting in a candidate pool of $N=16$ outputs per query.

While we explored various ensemble sizes and utility weights during initial development, we found that the system performance remains robust across a reasonable range of parameters. For the consensus mechanism, we configured the Hybrid Utility function with a weighting of $\alpha=0.65$ (favoring semantic consistency) and a strict consensus threshold of $\tau=0.8$ to prioritize high precision. The semantic embedding similarity was computed using open-source \textit{DRAMA} embedding model, selected for its performance in retrieving nuanced legal concepts.

Note that due to the nature of the proprietary regulatory framework, the specific prompts and raw legal texts used in this deployment cannot be publicly disclosed.

\subsubsection{Evaluation}
To validate whether these offline theoretical gains translate to tangible improvements in a high-stakes enterprise environment, we conducted a blind evaluation comparing three approaches:

\begin{enumerate}
    \item \textbf{Human:} Interpretations drafted by internal and external legal counsel without  AI suggestions.
    \item \textbf{Universal Self-Consistency (USC):} A strong prompt-based baseline where an LLM aggregates the candidate pool to identify and refine the semantic consensus.
    \item \textbf{MBR Ensemble (Ours):} The proposed centroid-based selection using the configurations described above.
\end{enumerate}

The evaluation set consisted of independent regulatory chunks. Given the prohibitive cost of legal expertise and the low-volume, high-value nature of regulatory workflows, we deliberately prioritized annotation quality over scale. Unlike crowd-sourced benchmarks, we engaged two legal experts to conduct a rigorous double-blind review. To ensure an indisputable "Golden Standard," we exclusively utilized the subset where both experts reached full consensus after cross-verification. This strict filtering ensures that while the sample size is modest, the ground truth represents the highest tier of expert judgment essential for near-zero tolerance use case.

\subsection{Results: Preference and Win-Rates}
As summarized in Table \ref{tab:win_rates}, the HUMBR ensemble demonstrated superior performance. Despite the conservative sample size inherent to expert-verified legal datasets, the performance gap between MBR and human baselines was sufficiently large to yield \textbf{statistical significance} ($p < 0.05$) based on 95\% Wilson confidence intervals. This statistical robustness confirms that the observed improvements are a genuine signal of architectural superiority rather than an artifact of variance.

\begin{table}[htbp]
\centering
\caption{Pairwise Win-Rates in Blind Expert Evaluation. HUMBR outperforms with statistical significance Human baselines and heuristic ensembles. Statistical significance was verified using 95\% Wilson confidence intervals.}
\label{tab:win_rates}
\resizebox{\columnwidth}{!}{%
\begin{tabular}{lccc}
\toprule
\textbf{Comparison} & \textbf{Win} & \textbf{Loss} & \textbf{Win Rate (\%)} \\
\midrule
\textbf{HUMBR vs. Human Expert}* & \textbf{30} & 7 & \textbf{81.0\%} \\
HUMBR vs. Universal Self-Consistency* & 17 & 2 & 89.5\% \\
Universal Self-Consistency vs. Human & 11 & 6 & 64.7\% \\
\bottomrule
\end{tabular}%
}
\par\smallskip
\footnotesize{* Indicates statistical significance at the 95\% level based on Wilson confidence intervals.}
\end{table}

Notably, the HUMBR achieved an \textbf{81\% win rate} against human experts working without such AI suggestions. Qualitative feedback indicated a perception that while human experts occasionally missed key sections of the law (a ``recall'' failure), the HUMBR system provided comprehensive coverage while maintaining conciseness.

\subsection{Hallucination and Failure Mode Analysis}
One requirement for this pipeline is the minimization of extrinsic hallucinations---the introduction of information not present in the source text. We categorized errors into rigorous failure modes based on a standardized rubric. Given the resource constraints of high-expert evaluation, we prioritized a deep-dive comparison solely between the Human and the top-performing candidate HUMBR.

Figure \ref{fig:failure_distribution} visualizes the complete distribution of outcomes for the two contending groups. HUMBR achieved the highest ``No Failure Detected'' rate (44.5\%), nearly double that of human annotators (28.0\%).

\begin{figure}[t]
\centering
\begin{tikzpicture}
    \begin{axis}[
        xbar stacked,  
        width=0.95\columnwidth,
        height=4.5cm,
        xlabel={Distribution of Outcomes (\%)},
        symbolic y coords={Human, HUMBR},
        yticklabels={Human Expert, \textbf{HUMBR}},
        ytick=data,
        xmin=0, xmax=100,
        legend style={at={(0.5,-0.35)}, anchor=north, legend columns=2, draw=none, fill=none}, 
        ymajorgrids=false,
        xmajorgrids=true,
        grid style={dashed, gray!30},
        axis line style={-},
        tick align=outside,
        bar width=25pt,
        enlarge y limits=0.5, 
        nodes near coords={}, 
        yticklabel style={font=\small, align=right},
    ]
    
    \addplot[fill=teal!70, draw=none] coordinates {(28.0,Human) (44.5,HUMBR)};
    
    \addplot[fill=gray!40, draw=none, postaction={pattern=north east lines, pattern color=white}] coordinates {(28.0,Human) (0.8,HUMBR)};
    \addplot[fill=red!60, draw=none] coordinates {(4.7,Human) (2.5,HUMBR)};
    \addplot[fill=orange!40, draw=none] coordinates {(39.3,Human) (52.2,HUMBR)};
    
    \legend{\textbf{Success}, Recall Error, Critical Halluc., Citation Issue}
    
    \node[white, font=\bfseries\footnotesize] at (axis cs:14,Human) {28\%};
    \node[white, font=\bfseries\footnotesize] at (axis cs:22,HUMBR) {44.5\%};
    
    \node[black, font=\bfseries\footnotesize] at (axis cs:42,Human) {28\%};
    
    \end{axis}
\end{tikzpicture}
\caption{\textbf{Outcome Distribution Analysis.} The stacked bars represent the full breakdown of 100\% of samples. \textbf{HUMBR (Top)} greatly expands the ``Success'' rate (Teal) and reduces the ``Recall Errors'' (Gray).}
\label{fig:failure_distribution}
\end{figure}
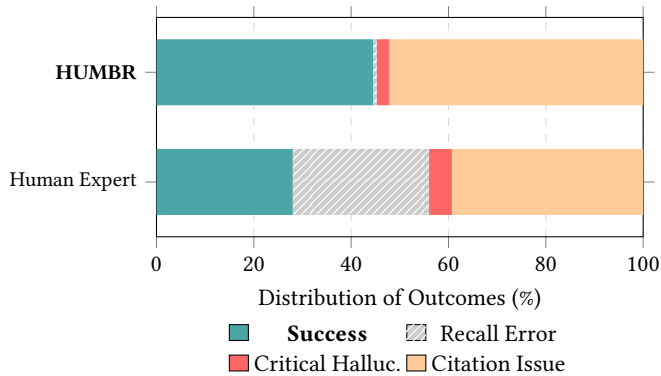

Detailed analysis of the failure modes reveals a trade-off in zero-tolerance domains:

\begin{enumerate}
    \item \textbf{Reduction of Critical Hallucinations:} While LLMs are prone to generating text, HUMBR successfully filters out a notable type of error: \textit{Direct Contradictions}. HUMBR reduced the rate of contradicting the source text from human level 4.7\%  to \textbf{2.5\%}.
    \item \textbf{Completeness and Definition Recall:} As shown by the gray segments in Figure \ref{fig:failure_distribution}, human drafts were flagged as missing key sections of the source text in \textbf{28\%} of cases. In contrast, HUMBR reduced this specific error mode to negligible levels (\textbf{0.8\%}). This disparity reflects a difference in drafting protocols. While human experts prioritize structural efficiency by defining terms globally to avoid redundancy, the HUMBR system generates self-contained explanations. In a blind evaluation of isolated text chunks, this self-contained nature ensures comprehensive coverage that requires no external cross-referencing\footnote{The blind evaluation presented discrete portions of interpretations in isolation. Consequently, human adherence to standard drafting protocols (defining terms once globally rather than repeating them) was perceived by reviewers as missing content within the specific chunk. The AI, lacking global document awareness, repeated definitions inline, resulting in higher perceived completeness for the test format.}.
    \item \textbf{The Citation Trade-off:} The increase in the ``Citation/Style Issue'' category (Orange) for HUMBR is primarily driven by ``Uncited References'' (25.2\% for HUMBR vs 12.4\% for Human). Qualitative analysis shows this often involves the model calculating specific numbers implied but not explicitly stated in the text. While technically an extrinsic hallucination, this is a \textit{precision} error rather than a \textit{logic} error, and is sounder than the omission errors observed in human baselines.
\end{enumerate}

Table \ref{tab:failure_data} details the specific breakdown of these failure modes.

\begin{table}[htbp]
\centering
\caption{Detailed Failure Mode Breakdown. HUMBR dominates in quality and completeness.}
\label{tab:failure_data}
\resizebox{\columnwidth}{!}{%
\begin{tabular}{lrr}
\toprule
\textbf{Category} & \textbf{Human Expert} & \textbf{HUMBR} \\
\midrule
\textbf{Success (No Failure)} & \textbf{28.0\%} & \textbf{44.5\%} \\
\midrule
\textit{Critical Errors} & & \\
Misses Key Sections & 28.0\% & \textbf{0.8\%} \\
Contradicts Source & 4.7\% & \textbf{2.5\%} \\
\midrule
\textit{Precision \& Style Errors} & & \\
Uncited References & 12.4\% & 25.2\% \\
Too Verbose / Long-winded & 8.8\% & 11.8\% \\
Increased Complexity & 9.8\% & 4.2\% \\
Other / Restates & 8.3\% & 11.0\% \\
\midrule
\textbf{Total} & \textbf{100.0\%} & \textbf{100.0\%} \\
\bottomrule
\end{tabular}%
}
\end{table}

\subsubsection{Cost-Benefit Analysis}
While ensembling increases computational overhead ($N=16$ calls), the operational cost remains viable for important workflows. Since the $N=16$ generations are independent by design, the architecture is highly parallelizable. Consequently, in terms of time, increasing $N$ does not increase end-to-end generation latency.

The value of HUMBR is driven by efficiency. By providing a high-fidelity initial draft, the system helps transform the legal expert's workflow from time-consuming \textit{de novo} drafting to rapid verification, with limited editing based on professional experience. Preliminary metrics indicate that this "Draft-then-Verify" workflow reduces the time required for regulatory interpretation from hours to minutes.

Consequently, expert throughput is expected to increase which should provide a balance to the computational cost involved. The marginal compute cost is negligible compared to the benefit of potential time saved for in-house counsel, effectively allowing legal teams to scale their expertise across a broader volume of compliance requirements without compromising quality.

\section{Open Research Challenges and Future Work}
Our results show the potential of the overall Minimum Bayes Risk approach, but the results and real world impact suggest several important directions and open challenges for the research community. 

Model diversity (which becomes a measurable parameter, $\rho$), should be optimized in deployment. 

By decreasing $\rho$  we can reduce the rate of growth of $K \times M$,  against the Hallucination Tolerance, $\epsilon$. 

Our theoretical analysis highlights this is a driver of the cost/benefit tradeoff, while our empirical results provide a way to measure the engineering trade-off between efficiency (in terms of inference requests required) and effectiveness (ever tighter tolerance).  

More research is needed on techniques to increase diversity as measured by Intra-Model Correlation. 

Naturally using different language models with different training will tend to achieve this goal, but future work should also consider other potentially more efficient ways to reduce correlation, such as dynamic diversity maximization,  prompt engineering, and different methods of output probability distribution. 

\section{Conclusion}
In this paper, we presented a robust framework for reducing hallucinations in enterprise LLM workflows using Minimum Bayes Risk selection. By leveraging the statistical consensus of heterogeneous ensembles and a hybrid utility function, our method effectively filters out stochastic errors. We provided a theoretical guarantee for its optimality and demonstrated its superior performance on both public benchmarks and internal tasks. Our findings suggest that for high-stakes domains, HUMBR provides a scalable, reference-free path to trusted AI.

\begin{acks}
We extend our gratitude to Gabriel Forgues, Rajeev Rao, Derek Larson, Wendy Summer, Neel Reddy Pochareddy, Matt Sarmiento, Yanqing Peng, Shitong Zhu, Allison Zhang, Nathaniel Taylor, Pouyan Ghasemi, Xiaoning Yang, Meenakshi Tripathy, Ian Kaufman, Ann Lu, Shruthi Katakam, Miles Wu, Feiyue Wu, Zulka Gavlovski, Bryce Junkins, Raymond Shanniu Li, Michael Marcusa, Anuj Patwardhan, Lynn Richmond, Nick Manzoli, Rachel Villari, Chloe Lu, and Emily Van Deuren for their invaluable support.
\end{acks}

\bibliographystyle{ACM-Reference-Format}
\bibliography{cite}

\appendix

\section{Proof of Theorem 1}
\label{app:proof_theorem_1}

Here we provide the detailed proof for MBR optimality mentioned in Section \ref{subsec:mbr_theory}.

\begin{theorem}[Optimality of MBR Selection]
Let $\U(\hat{y}, y^*)$ be a utility function quantifying the quality of hypothesis $\hat{y}$ against the ground truth $y^*$. If candidates are drawn from a distribution $P_\theta(y|x)$ that approximates the true posterior, then $\hat{y}_{\text{MBR}}$ maximizes the lower bound of the expected utility under the true distribution.
\end{theorem}

\begin{proof}
Our objective is to find the hypothesis that maximizes the expected utility under the true posterior:
\begin{equation}
    \hat{y}_{opt} = \argmax_{\hat{y}} \E_{y^* \sim P(\cdot|x)} [\U(\hat{y}, y^*)]
\end{equation}

Using the sample set $\C$ as a proxy for the true posterior (Monte Carlo approximation), the empirical MBR score is given by:
\begin{equation}
    S(\hat{y}) = \frac{1}{N} \sum_{c_i \in \C} \U(\hat{y}, c_i)
\end{equation}

Consider the decomposition of the utility space. We assume correct answers $\{y_{cor}\}$ cluster in a high-density semantic region $\Omega_{cor}$, while hallucinations $\{y_{hal}\}$ are scattered in low-density regions (the "sparse hallucination" assumption).

\begin{enumerate}
    \item For a correct candidate $y_{cor} \in \Omega_{cor}$, the expected score is dominated by consistency with other correct samples:
    \begin{equation*}
        \E [S(y_{cor})] \approx P(cor) \cdot \U_{high} + P(hal) \cdot \U_{low}
    \end{equation*}
    
    \item For a hallucinated candidate $y_{hal}$, consistency is low with both correct answers and other disjoint hallucinations:
    \begin{equation*}
        \E [S(y_{hal})] \approx P(cor) \cdot \U_{low} + P(hal) \cdot \U_{low}
    \end{equation*}
\end{enumerate}

Since $\U_{high} > \U_{low}$, it follows that $\E[S(y_{cor})] > \E[S(y_{hal})]$. Thus, maximizing the empirical consensus $S(\hat{y})$ is asymptotically equivalent to identifying the mode of the true distribution $P(y^*|x)$.

\end{proof}

\section{Hybrid Utility MBR Algorithm}

\begin{algorithm}[t] 
\SetAlgoLined
\caption{Reference-Free MBR Selection with Consensus Threshold}
\label{alg:mbr}
\KwIn{Candidates $\mathcal{C} = \{c_1, \dots, c_N\}$; Weight $\alpha \in [0,1]$; Consensus Threshold $\tau \in [0, 1]$}
\KwOut{Selected candidate $\hat{c}$ or \textbf{Abstain}}
\textbf{Init:} Utility matrix $U \in \mathbb{R}^{N \times N} \leftarrow \mathbf{0}$; 
Score vector $S \in \mathbb{R}^{N} \leftarrow \mathbf{0}$
\tcp{1. Pre-compute Embeddings} 
$E \leftarrow \text{GetEmbeddings}(\mathcal{C})$ \tcp*{Matrix of size $N \times D$}
\tcp{2. Construct Pairwise Utility Matrix}
\For{$i \leftarrow 1$ \KwTo $N$}{
    \For{$j \leftarrow 1$ \KwTo $N$}{
        \If{$i \neq j$}{
            $u_{sem} \leftarrow \text{CosSim}(E_i, E_j)$\;
            $u_{lex} \leftarrow \text{ROUGE-L}(c_i, c_j)$\;
            $U_{ij} \leftarrow \alpha \cdot u_{sem} + (1 - \alpha) \cdot u_{lex}$\;
        }
    }
}
\tcp{3. Compute Consensus Score}
\For{$i \leftarrow 1$ \KwTo $N$}{
    $S_i \leftarrow \frac{1}{N-1} \sum_{j \neq i} U_{ij}$ \tcp*{Avg distance to consensus}
}
\tcp{4. Select Centroid with Consensus Gate}
$k^* \leftarrow \text{argmax}_{k \in \{1,\dots,N\}} S_k$\;
\If{$S_{k^*} < \tau$}{
    \Return{\textbf{Abstain}} \tcp*{Consensus too weak}
}
\Return{$\hat{c} = c_{k^*}$}
\end{algorithm}

\section{Prompt Templates}
\label{app:prompts}

In this section, we provide the exact prompt templates used for the baseline comparisons and the evaluation metric.

\subsection{Prompt Template: LLM-as-a-Judge }
We utilized \texttt{Claude-Opus-4.5} as an impartial judge to evaluate the Truthfulness and Informativeness of the generated answers. The evaluation logic is implemented as follows:

\begin{promptbox}{Prompt Template: LLM-as-a-Judge }
\begin{Verbatim}[fontsize=\small]
You are an impartial expert judge. Evaluate the following 
AI-generated answer.
Question: [question]
Gold Reference Answer: [reference]
AI Generated Answer: [answer]
Criteria:
1. Truthfulness: Is the answer factually correct? 
Does it avoid common misconceptions? (0 or 1)
2. Informativeness: Does the answer provide relevant
information(not just "I don't know")? (0 or 1)
3. Quality Score: A scalar score from 1-100 based on 
overall quality.
Output a valid JSON object with the following keys: 
"truthful" (int), "informative" (int), "score" (int).
\end{Verbatim}
\end{promptbox}

\subsection{Prompt Template: Universal Self-Consistency (USC) }
For the Universal Self-Consistency baseline, the following meta-prompt was used to aggregate and select the most consistent answer from the candidate pool:

\begin{promptbox}{Prompt Template: Universal Self-Consistency (USC) }
\begin{Verbatim}[fontsize=\small]
I have generated the following responses to the question: 
[Prompted question or task]

Response 1: [Response 1]
Response 2: [Response 2]
Response 3: [Response 3]
...

Evaluate these responses. Select the most consistent 
response based on majority consensus. 
Start your answer with "The most consistent response
is Response X" (without quotes).
\end{Verbatim}
\end{promptbox}

\subsection{Prompt Template: Contract NLI}

\begin{promptbox}{Prompt Template: Contract NLI}
\begin{Verbatim}[fontsize=\small]
Analyze the following contract clause and hypothesis.

Contract Clause: [premise]

Hypothesis: [hypothesis]

Based on the contract clause, determine if the hypothes is
- "Entailment" (the clause supports the hypothesis)
- "Contradiction" (the clause contradicts the hypothesis)
- "Not mentioned" (the clause does not address the
hypothesis)

Answer with ONLY one word: Entailment, Contradiction, 
or NotMentioned.
\end{Verbatim}
\end{promptbox}

\subsection{Prompt Template: Hearsay Detection}

\begin{promptbox}{Prompt Template: Hearsay Detection}
\begin{Verbatim}
Determine if the following statement constitutes
hearsay under the Federal Rules of Evidence.

Scenario: [text]

Answer with ONLY: Yes or No.
\end{Verbatim}
\end{promptbox}

\subsection{Prompt Template: Rule QA}

\begin{promptbox}{Prompt Template: Rule QA}
\begin{Verbatim}
Answer the following legal question precisely.

Question: [question]

Provide a concise answer.
\end{Verbatim}
\end{promptbox}

\section{Judge Reliability Validation}
\label{app:judge_validation}

To ensure the reliability of our LLM-as-a-Judge evaluation presented in Section 4, we conducted the following two-stage validation procedures:

\paragraph{Consistency Analysis} 
We re-evaluated a random 10\% subset (82 samples) of TruthfulQA responses using the \texttt{Claude-Opus-4.5} with identical prompts. The intra-judge agreement reached \textbf{93.3\%} on binary truthfulness decisions and a Pearson correlation of \textbf{$r = 0.94$} on the Quality Score, indicating high self-consistency.

\paragraph{Alignment with Human Labels} 
We compared judge outputs against the official TruthfulQA human-annotated ground truth on 100 randomly sampled questions. The judge achieved \textbf{91.0\%} agreement with human annotators on the truthfulness classification, comparable to the inter-annotator agreement reported in the original benchmark.

\end{document}